\documentclass{article} 
\usepackage{nfam2026_workshop}
\usepackage{times}

\usepackage{amsmath,amsfonts,bm}

\def\eqref#1{equation~\ref{#1}}

\def\1{\bm{1}}

\DeclareMathAlphabet{\mathsfit}{\encodingdefault}{\sfdefault}{m}{sl}
\SetMathAlphabet{\mathsfit}{bold}{\encodingdefault}{\sfdefault}{bx}{n}

\newcommand{\softmax}{\mathrm{softmax}}

\usepackage{hyperref}
\usepackage{url}
\usepackage{booktabs}       
\usepackage{amsmath,amssymb,amsthm}
\usepackage{mathtools}
\usepackage{graphicx}
\usepackage{subcaption}
\usepackage{xcolor}
\usepackage{multirow}

\newcommand{\ours}{\textsc{GHN}}
\newcommand{\lse}{\mathrm{lse}}

\newcommand{\ReLU}{\mathrm{ReLU}}
\newcommand{\tr}{\mathrm{tr}}
\newcommand{\diag}{\mathrm{diag}}

\newtheorem{theorem}{Theorem}[section]
\newtheorem{lemma}[theorem]{Lemma}
\newtheorem{proposition}[theorem]{Proposition}
\newtheorem{corollary}[theorem]{Corollary}
\theoremstyle{remark}
\newtheorem{remark}[theorem]{Remark}

\title{Graph Hopfield Networks: \\ Energy-Based Node Classification with Associative Memory}

\author{Abinav Rao$^{*}$ \And Alex Wa$^{*}$ \And Rishi Athavale$^{*}$ \\\\ \small $^{*}$Equal contribution}

\iclrfinalcopy
\begin{document}

\maketitle

\begin{abstract}
We introduce \emph{Graph Hopfield Networks} (\ours{}), whose energy function couples associative memory retrieval with graph Laplacian smoothing for node classification.
Gradient descent on this joint energy yields an iterative update interleaving Hopfield retrieval with Laplacian propagation.
Memory retrieval provides regime-dependent benefits: up to 2.0~pp on sparse citation networks and up to 5~pp additional robustness under feature masking; the iterative energy-descent architecture itself is a strong inductive bias, with all variants (including the memory-disabled NoMem ablation) outperforming standard baselines on Amazon co-purchase graphs.
Tuning $\lambda \leq 0$ enables graph sharpening for heterophilous benchmarks without architectural changes.
\end{abstract}

\section{Introduction}
\label{sec:intro}

Modern Hopfield networks~\citep{krotov2016dense, demircigil2017model, ramsauer2020hopfield} have renewed interest in associative memory as a computational primitive.
Their connection to Transformer attention~\citep{ramsauer2020hopfield} has spurred applications across vision, language, and scientific domains.
Yet despite growing theoretical understanding (including hierarchical memory organization~\citep{krotov2021hierarchical}, novel energy functions~\citep{hoover2025dense}, and capacity analyses~\citep{lucibello2024exponential}), their application to \emph{node classification} on graphs via an energy-based formulation remains, to our knowledge, underexplored. \citet{liang2022modern} apply modern Hopfield networks to graph-level tasks, but do not couple memory retrieval to graph structure in a joint energy.

Graph neural networks (GNNs) learn node representations by aggregating neighborhood information~\citep{kipf2017semi, velickovic2018graph, hamilton2017inductive}, but degrade under noisy or incomplete edges ~\citep{zugner2018adversarial}.
Associative memory offers an alternative, content-based signal: rather than relying on local structure, it retrieves relevant patterns from feature content.

We propose \emph{Graph Hopfield Networks} (\ours{}), which combine these two mechanisms in a single energy function:
\begin{equation}
\label{eq:energy}
E_{\mathrm{GH}}(\mathbf{X}) = \sum_{v \in \mathcal{V}} \left[ -\lse(\beta, \mathbf{M} \mathbf{x}_v) + \tfrac{1}{2}\|\mathbf{x}_v\|^2 \right] + \lambda \, \tr(\mathbf{X}^\top \mathbf{L} \mathbf{X}),
\end{equation}
where $\lse(\beta, \mathbf{z}) \coloneqq \beta^{-1}\log\sum_\mu \exp(\beta z_\mu)$ is the log-sum-exp operator.
The first term drives memory retrieval toward learned patterns $\mathbf{M}$ and the second encourages smoothness over the graph Laplacian $\mathbf{L}$.
Gradient descent on this energy yields an update rule interleaving Hopfield retrieval with Laplacian propagation (Section~\ref{sec:method}; derivation in Appendix~\ref{app:theory_diff}).
The iterative energy-descent architecture is itself a strong inductive bias: all \ours{} variants,including a memory-disabled ablation, outperform every standard baseline on Amazon co-purchase graphs (Section~\ref{sec:results}).
Memory acts as a substitute for missing structural signal: it adds up to 2.0~pp on sparse citation networks and up to 5~pp robustness under feature masking (Section~\ref{sec:robustness}), but is redundant on dense graphs where the Laplacian suffices.
Tuning $\lambda \leq 0$ enables graph sharpening for heterophilous benchmarks without architectural changes (Section~\ref{sec:heterophily}).

\section{Graph Hopfield Networks}
\label{sec:method}

\subsection{Energy Function and Update Rule}
\label{sec:energy}

Given a graph $\mathcal{G} = (\mathcal{V}, \mathcal{E})$ with node features $\mathbf{X} \in \mathbb{R}^{N \times d}$ and (symmetric normalized) graph Laplacian $\mathbf{L} = \mathbf{I} - \mathbf{D}^{-\frac1{2}}\mathbf{A}\mathbf{D}^{-\frac1{2}}$, we define the \emph{Graph Hopfield energy} as in Eq.~\ref{eq:energy}.
The two terms encode complementary objectives: the Hopfield term drives each node's representation toward relevant learned patterns, while the Laplacian term encourages neighboring nodes to have similar representations.

\begin{figure}[h]
\centering
\includegraphics[width=0.85\linewidth]{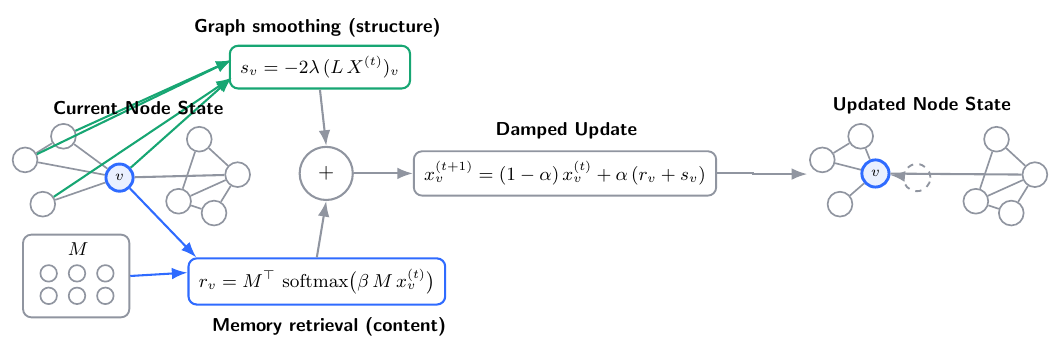}
\vspace{-4pt}
\caption{One iteration of the \ours{} update: node features $\mathbf{x}_v$ are blended with (i) memory retrieval from pattern bank $\mathbf{M}$, and (ii) graph Laplacian smoothing over neighbors. Damping $\alpha$ controls the step size.}
\label{fig:ghn_update}
\end{figure}

Taking the gradient with respect to $\mathbf{x}_v$ yields the update rule:
\begin{equation}
\label{eq:update}
\mathbf{x}_v^{(t+1)} = (1-\alpha)\,\mathbf{x}_v^{(t)} + \alpha \left[ \underbrace{\mathbf{M}^\top \softmax\!\left(\beta \mathbf{M} \mathbf{x}_v^{(t)}\right)}_{\text{memory retrieval}} - 2\lambda \underbrace{(\mathbf{L}\mathbf{X}^{(t)})_v}_{\text{graph smoothing}} \right]
\end{equation}
where $\alpha \in (0,1)$ is a damping coefficient.
This naturally interleaves two operations per iteration: memory retrieval from the memory bank $\mathbf{M}$, and structure-based smoothing via the graph Laplacian.

For least-squares retrieval (LSR; \citet{hoover2025dense}), the softmax is replaced by normalized Epanechnikov kernel weights: $w_\mu = \ReLU(1 - \frac{\beta}{2}\|\mathbf{x}_v - \mathbf{m}_\mu\|^2) / \sum_{\mu'} \ReLU(1 - \frac{\beta}{2}\|\mathbf{x}_v - \mathbf{m}_{\mu'}\|^2)$.
In our implementation, $\beta$ is treated as a learnable parameter, so retrieval sharpness is adapted during training.

Convergence of the base LSE formulation is analyzed in Appendix~\ref{app:theory}; in practice we observe stable convergence within $T{=}4$ iterations across all experiments.

\subsection{Gated Memory Retrieval}
\label{sec:gate}

Memory retrieval can produce poor outputs early in training or when queries fall far from any stored pattern; to prevent this from corrupting representations, we introduce a learned gate:
\begin{equation}
\label{eq:gate}
\mathbf{g}_v = \sigma\!\left(\mathbf{W}_g [\mathbf{x}_v \| \mathbf{r}_v] + \mathbf{b}_g\right), \qquad
\tilde{\mathbf{r}}_v = \mathbf{g}_v \odot \mathbf{r}_v + (1 - \mathbf{g}_v) \odot \mathbf{x}_v,
\end{equation}
where $\mathbf{r}_v$ is the raw retrieval output, $\sigma$ is the sigmoid, $\|$ denotes concatenation, and $\odot$ is element-wise multiplication.
We initialize $b_g = 2$ ($\mathbf{g}_v \approx 0.88$) to encourage memory use early in training; the gated output $\tilde{\mathbf{r}}_v$ replaces the retrieval term in Eq.~\ref{eq:update}.

\subsection{Hierarchical Memory}
\label{sec:hierarchical}

Flat attention over $K$ patterns can be unstable when $K$ is large relative to the feature dimension (Section~\ref{sec:results}).
Inspired by~\citet{krotov2021hierarchical}, we partition $K$ patterns into $G$ groups and use two-stage retrieval: (i) \textbf{routing}, soft-assign queries to groups via centroids computed from pattern parameters, then (ii) \textbf{retrieval}, standard Hopfield retrieval within each selected group. The full model is an encoder, $L$ stacked \ours{} layers (each with $T$ iterations and gating), and a linear classifier; retrieval uses tied keys/values.
Details in Appendix~\ref{app:details}.

\section{Experiments}
\label{sec:experiments}

\subsection{Setup}

\textbf{Datasets.}
We evaluate on nine benchmarks: three homophilous citation networks (Cora, CiteSeer, PubMed~\citep{sen2008collective}), two Amazon co-purchase graphs (Photo and Computers~\citep{shchur2018pitfalls}), and four heterophilous graphs (Texas, Wisconsin, Cornell~\citep{pei2020geom}, and Actor~\citep{tang2009social}).
Standard public splits are used for Planetoid; random 60/20/20 splits for Amazon; standard splits for the heterophilous datasets.

\textbf{Models.}
We compare \ours{} variants (\textbf{LSE} (flat, softmax retrieval), \textbf{LSR} (flat, Epanechnikov retrieval), \textbf{Hier($G$)} (hierarchical with $G$ groups), and \textbf{NoMem} (memory-disabled ablation, Laplacian smoothing only)) against GCN~\citep{kipf2017semi}, GAT~\citep{velickovic2018graph}, GraphSAGE~\citep{hamilton2017inductive}, APPNP~\citep{klicpera2019predict}, GIN~\citep{xu2019powerful}, MLP, and GPR-GNN~\citep{chien2021adaptive} (heterophily-specific; results from the original paper on the same evaluation protocol).

\textbf{Training.}
Per-model hyperparameter tuning on each dataset; each model selects its best config by validation accuracy.
All results are mean $\pm$ std over 10 seeds.
Full training protocol and hyperparameter grids are in Appendix~\ref{app:details}.

\subsection{Node Classification}
\label{sec:results}

\begin{table}[h]
\centering
\caption{Node classification accuracy (\%) with per-model HP tuning. \textbf{Bold}: best per column. $\dagger$: bimodal ($>$10\% std).}
\label{tab:accuracy}
\small
\begin{tabular}{@{}lccccc@{}}
\toprule
& Cora & CiteSeer & PubMed & Photo & Computers \\
\midrule
\multicolumn{6}{@{}l}{\textit{\ours{} variants}} \\
LSE         & $82.0 {\scriptstyle\pm 0.6}$ & $70.9 {\scriptstyle\pm 0.8}$ & $77.9 {\scriptstyle\pm 0.4}$ & $94.5 {\scriptstyle\pm 0.3}$ & $91.3 {\scriptstyle\pm 0.2}$ \\
LSR         & $82.0 {\scriptstyle\pm 0.6}$ & $70.9 {\scriptstyle\pm 1.0}$ & $78.1 {\scriptstyle\pm 0.4}$ & $94.4 {\scriptstyle\pm 0.2}$ & $91.3 {\scriptstyle\pm 0.4}$ \\
Hier(8)     & $82.3 {\scriptstyle\pm 0.4}$ & $70.3 {\scriptstyle\pm 1.1}$ & $77.6 {\scriptstyle\pm 0.8}$ & $94.2 {\scriptstyle\pm 0.2}$ & $91.3 {\scriptstyle\pm 0.2}$ \\
NoMem       & $80.3 {\scriptstyle\pm 0.4}$ & $69.8 {\scriptstyle\pm 1.1}$ & $77.2 {\scriptstyle\pm 0.6}$ & $\mathbf{94.7 {\scriptstyle\pm 0.2}}$ & $\mathbf{91.5 {\scriptstyle\pm 0.2}}$ \\
\midrule
\multicolumn{6}{@{}l}{\textit{Baselines}} \\
GCN         & $82.2 {\scriptstyle\pm 0.4}$ & $72.2 {\scriptstyle\pm 0.5}$ & $79.4 {\scriptstyle\pm 0.4}$ & $89.6 {\scriptstyle\pm 0.3}$ & $73.7 {\scriptstyle\pm 13.0}^\dagger$ \\
GAT         & $\mathbf{82.8 {\scriptstyle\pm 0.8}}$ & $72.3 {\scriptstyle\pm 0.6}$ & $78.2 {\scriptstyle\pm 0.4}$ & $92.2 {\scriptstyle\pm 0.3}$ & $87.9 {\scriptstyle\pm 0.6}$ \\
GraphSAGE   & $80.2 {\scriptstyle\pm 0.6}$ & $70.5 {\scriptstyle\pm 0.8}$ & $77.9 {\scriptstyle\pm 0.4}$ & $93.9 {\scriptstyle\pm 0.2}$ & $87.4 {\scriptstyle\pm 0.4}$ \\
APPNP       & $\mathbf{82.8 {\scriptstyle\pm 0.7}}$ & $\mathbf{72.4 {\scriptstyle\pm 0.5}}$ & $\mathbf{80.3 {\scriptstyle\pm 0.2}}$ & $92.2 {\scriptstyle\pm 0.3}$ & $41.9 {\scriptstyle\pm 13.1}^\dagger$ \\
GIN         & $77.5 {\scriptstyle\pm 1.2}$ & $63.8 {\scriptstyle\pm 2.3}$ & $77.6 {\scriptstyle\pm 0.7}$ & $93.0 {\scriptstyle\pm 0.5}$ & $88.9 {\scriptstyle\pm 0.5}$ \\
MLP         & $60.0 {\scriptstyle\pm 0.8}$ & $60.6 {\scriptstyle\pm 0.6}$ & $73.4 {\scriptstyle\pm 0.5}$ & $88.5 {\scriptstyle\pm 0.3}$ & $75.4 {\scriptstyle\pm 11.9}^\dagger$ \\
\bottomrule
\end{tabular}
\end{table}

\textbf{Iterative architecture is the primary driver on Amazon.}
Every \ours{} variant, including the memory-disabled NoMem ablation, outperforms the best standard baseline on both Amazon datasets (+0.8~pp on Photo, +2.6~pp on Computers).
GCN, APPNP, and MLP exhibit bimodal training collapse on Computers ($>$10\% std), a known failure mode that the iterative energy-descent architecture avoids entirely.
Memory-enabled variants are within 0.5~pp of NoMem on Amazon, suggesting that on dense graphs (Photo: 238k edges vs.\ Cora: 10.6k), Laplacian smoothing captures the relevant signal and memory provides marginal additional value.

\textbf{Memory retrieval adds value on sparse citation graphs.}
On all three Planetoid datasets, the best memory-enabled variant outperforms NoMem: Hier(8) on Cora (+2.0~pp, non-overlapping $\pm$1$\sigma$), LSR on CiteSeer (+1.1~pp), and LSR on PubMed (+0.9~pp), though the latter two gaps overlap in standard deviation.
\ours{} variants are competitive with but do not surpass GAT or APPNP on Planetoid.

\textbf{LSR is more stable than LSE at moderate $K$.}
A $\beta \times K$ phase diagram (Figure~\ref{fig:phase_diagram}) shows LSE requires $K{=}256$ to converge on Amazon Photo, while LSR is stable at $K{=}64$ for moderate $\beta$~\citep{hoover2025dense}.

\subsection{Robustness to Graph Corruption}
\label{sec:robustness}

We test three corruption types on Amazon Photo: edge removal, feature masking, and additive feature noise.

\begin{table}[h]
\centering
\caption{Accuracy (\%) under 50\% corruption on Amazon Photo. Clean baselines may differ from Table~\ref{tab:accuracy} due to a separate HP sweep optimized for robustness. \textbf{Bold}: best per column.}
\label{tab:amazon_corruption}
\small
\begin{tabular}{@{}lcccc@{}}
\toprule
& Clean & Edge drop & Feat.\ mask & Feat.\ noise \\
\midrule
\multicolumn{5}{@{}l}{\textit{\ours{} variants}} \\
LSE         & $\mathbf{94.4 {\scriptstyle\pm 0.3}}$ & $93.7 {\scriptstyle\pm 0.4}$ & $89.4 {\scriptstyle\pm 1.8}$ & $93.9 {\scriptstyle\pm 0.3}$ \\
LSR         & $94.3 {\scriptstyle\pm 0.2}$ & $93.7 {\scriptstyle\pm 0.5}$ & $91.0 {\scriptstyle\pm 2.4}$ & $94.1 {\scriptstyle\pm 0.2}$ \\
Hier(8)     & $94.2 {\scriptstyle\pm 0.2}$ & $93.7 {\scriptstyle\pm 0.5}$ & $\mathbf{91.9 {\scriptstyle\pm 1.1}}$ & $93.9 {\scriptstyle\pm 0.1}$ \\
NoMem       & $\mathbf{94.4 {\scriptstyle\pm 0.3}}$ & $\mathbf{94.0 {\scriptstyle\pm 0.5}}$ & $86.9 {\scriptstyle\pm 3.1}$ & $\mathbf{94.2 {\scriptstyle\pm 0.3}}$ \\
\midrule
\multicolumn{5}{@{}l}{\textit{Baselines}} \\
GCN         & $89.7 {\scriptstyle\pm 0.3}$ & $89.1 {\scriptstyle\pm 0.5}$ & $81.2 {\scriptstyle\pm 1.5}$ & $89.7 {\scriptstyle\pm 0.3}$ \\
GAT         & $92.2 {\scriptstyle\pm 0.3}$ & $91.2 {\scriptstyle\pm 0.5}$ & $81.4 {\scriptstyle\pm 5.0}$ & $92.2 {\scriptstyle\pm 0.2}$ \\
GraphSAGE   & $93.9 {\scriptstyle\pm 0.2}$ & $93.0 {\scriptstyle\pm 0.3}$ & $73.3 {\scriptstyle\pm 1.8}$ & $93.6 {\scriptstyle\pm 0.3}$ \\
APPNP       & $92.2 {\scriptstyle\pm 0.3}$ & $91.7 {\scriptstyle\pm 0.3}$ & $65.0 {\scriptstyle\pm 1.8}$ & $91.8 {\scriptstyle\pm 0.3}$ \\
GIN         & $92.6 {\scriptstyle\pm 0.4}$ & $89.6 {\scriptstyle\pm 0.8}$ & $91.1 {\scriptstyle\pm 0.9}$ & $92.6 {\scriptstyle\pm 0.5}$ \\
MLP         & $88.5 {\scriptstyle\pm 0.3}$ & $88.5 {\scriptstyle\pm 0.3}$ & $70.7 {\scriptstyle\pm 0.6}$ & $86.6 {\scriptstyle\pm 0.2}$ \\
\bottomrule
\end{tabular}
\end{table}

All memory-enabled \ours{} variants rank in the top four under every corruption type, while NoMem joins them except under feature masking, where GIN's structure-independent aggregation (91.1\%) outperforms NoMem (86.9\%) (Table~\ref{tab:amazon_corruption}).
Under edge removal and feature noise, NoMem performs comparably to memory-enabled variants, indicating that iterative Laplacian propagation is the primary robustness mechanism.
\textbf{Under feature masking, memory adds substantial value.}
Hier(8) retains 91.9\% at 50\% masking versus 86.9\% for NoMem, a 5.0~pp gap absent on clean data.
This is the clearest evidence that stored patterns compensate for degraded node features, a mechanism unavailable to structure-only smoothing.

On Planetoid, \ours{} robustness varies by dataset (LSE leads on PubMed; details in Appendix~\ref{app:edge_robustness}); Figure~\ref{fig:robustness} shows edge-deletion curves.

\begin{figure}[h]
\centering
\includegraphics[width=0.9\linewidth]{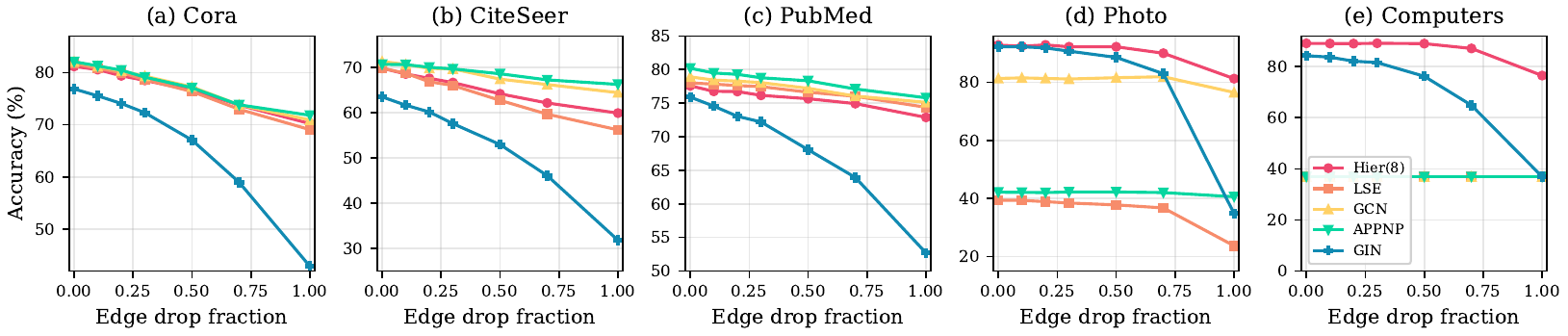}
\vspace{-4pt}
\caption{Accuracy under progressive edge deletion on Amazon and Planetoid datasets. \ours{} variants degrade more gradually on Amazon; Planetoid details in Appendix~\ref{app:edge_robustness}.}
\label{fig:robustness}
\end{figure}

\subsection{Heterophilous Graphs}
\label{sec:heterophily}

On heterophilous graphs, Laplacian smoothing ($\lambda>0$) pushes neighbors toward similar representations and can harm classification. We tune $\lambda \in \{-0.1, -0.05, -0.01, 0\}$ to enable graph sharpening (negative $\lambda$ pushes neighbors apart).

\begin{table}[h]
\centering
\caption{Node classification (\%) on heterophilous benchmarks. \textbf{Bold}: best per column. $\dagger$Published results~\citep{chien2021adaptive} on same splits; Actor not reported.}
\label{tab:heterophilous}
\small
\begin{tabular}{@{}lcccc@{}}
\toprule
& Texas & Wisconsin & Cornell & Actor \\
\midrule
\multicolumn{5}{@{}l}{\textit{\ours{} variants}} \\
LSE          & $\mathbf{81.4 {\scriptstyle\pm 1.5}}$ & $82.2 {\scriptstyle\pm 0.6}$ & $78.1 {\scriptstyle\pm 3.2}$ & $\mathbf{38.0 {\scriptstyle\pm 0.5}}$ \\
LSR          & $80.8 {\scriptstyle\pm 1.5}$ & $82.5 {\scriptstyle\pm 1.1}$ & $78.9 {\scriptstyle\pm 1.7}$ & $37.9 {\scriptstyle\pm 0.5}$ \\
Hier(8)      & $80.5 {\scriptstyle\pm 3.3}$ & $82.2 {\scriptstyle\pm 1.7}$ & $78.1 {\scriptstyle\pm 2.0}$ & $37.8 {\scriptstyle\pm 0.6}$ \\
NoMem        & $80.5 {\scriptstyle\pm 1.7}$ & $82.4 {\scriptstyle\pm 0.9}$ & $78.6 {\scriptstyle\pm 2.4}$ & $\mathbf{38.0 {\scriptstyle\pm 0.7}}$ \\
\midrule
\multicolumn{5}{@{}l}{\textit{Heterophily-aware baselines}} \\
GPR-GNN$^\dagger$ & $78.4 {\scriptstyle\pm 4.4}$ & $\mathbf{82.9 {\scriptstyle\pm 4.2}}$ & $\mathbf{80.3 {\scriptstyle\pm 8.1}}$ & -- \\
\midrule
\multicolumn{5}{@{}l}{\textit{Standard baselines}} \\
MLP          & $80.3 {\scriptstyle\pm 1.8}$ & $82.2 {\scriptstyle\pm 0.6}$ & $73.0 {\scriptstyle\pm 0.0}$ & $37.6 {\scriptstyle\pm 0.4}$ \\
GraphSAGE    & $78.4 {\scriptstyle\pm 1.8}$ & $80.0 {\scriptstyle\pm 0.8}$ & $76.2 {\scriptstyle\pm 1.7}$ & $37.0 {\scriptstyle\pm 0.5}$ \\
GCN          & $70.3 {\scriptstyle\pm 1.3}$ & $55.7 {\scriptstyle\pm 1.0}$ & $40.3 {\scriptstyle\pm 0.9}$ & $28.8 {\scriptstyle\pm 0.4}$ \\
GAT          & $68.6 {\scriptstyle\pm 2.3}$ & $58.0 {\scriptstyle\pm 1.9}$ & $43.8 {\scriptstyle\pm 4.0}$ & $29.4 {\scriptstyle\pm 0.6}$ \\
APPNP        & $70.0 {\scriptstyle\pm 1.5}$ & $56.7 {\scriptstyle\pm 1.1}$ & $43.2 {\scriptstyle\pm 3.6}$ & $32.4 {\scriptstyle\pm 0.6}$ \\
GIN          & $64.9 {\scriptstyle\pm 0.0}$ & $50.4 {\scriptstyle\pm 1.6}$ & $55.7 {\scriptstyle\pm 6.3}$ & $27.3 {\scriptstyle\pm 0.5}$ \\
\bottomrule
\end{tabular}
\end{table}

Table~\ref{tab:heterophilous} shows that with tuned negative $\lambda$, \ours{} variants perform competitively with GPR-GNN, a method specifically designed for heterophily, while substantially outperforming standard GNNs (GCN, GAT, APPNP), which all collapse under graph smoothing.
We report published GPR-GNN results on the same splits ($\dagger$); we did not re-tune or re-run GPR-GNN.
\ours{}-LSE leads on Texas (81.4\% vs.\ 78.4\%), while GPR-GNN leads on Wisconsin (82.9\% vs.\ 82.5\%) and Cornell (80.3\% vs.\ 78.9\%); all differences are within overlapping standard deviations.
Notably, GPR-GNN exhibits much higher variance (e.g., $\pm$8.1 on Cornell vs.\ $\pm$1.7 for \ours{}-LSR), suggesting that \ours{} is more stable across splits on these small datasets.
The key advantage is architectural: \ours{} handles both homophilous and heterophilous graphs by tuning a single parameter $\lambda$, whereas GPR-GNN requires learning $K$ polynomial filter coefficients specialized for the graph's spectral properties.

All best configs select $\lambda \in \{-0.05, -0.01\}$; Appendix~\ref{app:lambda_ablation} shows a sharp phase transition (e.g., LSE on Cornell: 73.5\% at $\lambda{=}0$ to 44.3\% at $\lambda{=}0.3$).
Memory-enabled and NoMem variants perform similarly at $\lambda \leq 0$, since the Laplacian term is minimal and all variants reduce to feature-space operations.

\section{Discussion and Future Directions}
\label{sec:conclusion}

The central finding of this work is a decomposition: on graphs, \emph{iterative energy descent matters more than what you descend on}.
NoMem, an energy-descent architectural ablation that removes the Hopfield term entirely, matches or exceeds memory-enabled variants on both Amazon datasets and ties on Actor, yet still outperforms every standard baseline on Amazon.
This suggests that the iterative damped-update architecture provides a strong inductive bias that stabilizes training on graphs where feedforward baselines collapse and enables graceful degradation under edge corruption.
For associative memory, the regime-dependent benefits are key: memory adds value on sparse graphs and under feature masking, while the ablation clarifies the respective roles of architecture vs.\ retrieval.

Memory's contribution is better understood as \emph{substitutive} rather than complementary to graph structure.
On dense graphs (Photo: 238k edges), the Laplacian already provides sufficient signal for classification, and memory is redundant.
On sparse graphs (Cora: 10.6k edges) or under feature corruption, the structural signal is insufficient and memory fills the gap, up to 2.0~pp on clean Cora and 5.0~pp under 50\% feature masking.
This substitution view explains why the gate learns a near-constant blending ratio (Appendix~\ref{app:gate_analysis}): the optimal balance between memory and structure is a property of the \emph{graph regime}, not of individual nodes, so a static gate suffices.

\textbf{Limitations.}
\ours{} does not outperform GAT/APPNP on clean Planetoid graphs; on heterophilous benchmarks, gains over MLP are modest and often within one standard deviation.
Flat LSE/LSR require $K{=}256$ to avoid collapse on Amazon; we evaluate only random (not adversarial) corruption; $O(NK)$ retrieval adds overhead (approximately $1.5$--$2{\times}$ wall-clock time per epoch relative to GAT, estimated across datasets).
The convergence theory (Appendix~\ref{app:theory}) requires $\beta\|\mathbf{M}\|^2 < 2$, a condition violated at trained operating points for most datasets (Table~\ref{tab:beta_norm}); in practice, stability is achieved empirically within $T{=}4$ iterations.
Promising extensions include per-node $\lambda$, replacing the gate with a fixed scalar blend, adversarial training of memory patterns, sparse retrieval for scaling, and comparison with graph transformers, the most direct architectural competitors.

\bibliographystyle{nfam2026_workshop}
\bibliography{references}


\newpage
\appendix

\section{Experimental Details}
\label{app:details}

Table~\ref{tab:hyperparams} gives the full hyperparameter grid.
The full model is encoder $\to$ $L$ Graph Hopfield layers (each with $T$ iterations and gating) $\to$ classifier; retrieval uses tied keys/values, skip connections (0.1), dropout, and LayerNorm.
Model-specific tuning axes: GAT heads $\in \{$4, 8$\}$, APPNP $\alpha \in \{$0.1, 0.2$\}$, \ours{}-LSE/LSR $K \in \{$64, 256$\}$.
\ours{} uses $\lambda{=}0.3$ on homophilous datasets and $\lambda \in \{-0.1, -0.05, -0.01, 0\}$ on heterophilous datasets.

\begin{table}[h]
\centering
\caption{Hyperparameters for \ours{} and training protocol. Per-model HP tuning (Table~\ref{tab:accuracy}) selects the best config per dataset from the grid ranges shown.}
\label{tab:hyperparams}
\small
\begin{tabular}{@{}ll@{}}
\toprule
\textbf{Parameter} & \textbf{Value / Grid} \\
\midrule
\multicolumn{2}{@{}l}{\textit{\ours{} (all variants)}} \\
Hidden dimension & $\{64, 128\}$ \\
Number of patterns $K$ & $\{64, 256\}$ (LSE/LSR) \\
Inverse temperature $\beta$ (init) & 1.0 (learnable) \\
Laplacian weight $\lambda$ & 0.3 \\
Damping $\alpha$ & 0.3 \\
Iterations per layer $T$ & 4 \\
Number of layers $L$ & 2 \\
Number of groups $G$ (Hier) & 8 \\
Memory heads $H$ & $\{1, 2, 4, 8\}$ (ablated; Appendix~\ref{app:heads_ablation}) \\
Dropout & $\{0.3, 0.5\}$ \\
Gate bias init & 2.0 \\
Keys/values & Tied \\
Skip weight & 0.1 \\
\midrule
\multicolumn{2}{@{}l}{\textit{Training (all models)}} \\
Learning rate & $\{0.001, 0.005, 0.01\}$ \\
Weight decay & $\{10^{-4}, 5{\times}10^{-4}, 10^{-3}\}$ \\
Optimizer & Adam \\
Epochs & 300 \\
Early stopping patience & 50 \\
\bottomrule
\end{tabular}
\end{table}

\textbf{Baseline collapse on Amazon Computers.}
The bimodal collapse of GCN, APPNP, and MLP on Amazon Computers (Table~\ref{tab:accuracy}) persists across our full hyperparameter grid: we swept learning rates $\{0.001, 0.005, 0.01\}$, weight decay $\{10^{-4}, 5{\times}10^{-4}, 10^{-3}\}$, and dropout $\{0.3, 0.5\}$, with no configuration eliminating the failure mode.
Collapse manifests as some seeds converging normally (${\sim}87\%$) while others diverge to near-chance (${\sim}50\%$), producing the $>$10\% standard deviations reported.
This instability is consistent with known sensitivity of spectral-normalization-based propagation (GCN, APPNP) on graphs with heterogeneous degree distributions~\citep{shchur2018pitfalls}; Amazon Computers has a heavy-tailed degree distribution (max degree $>$400) that amplifies initialization-dependent convergence.
We did not apply additional stabilization techniques (e.g., batch normalization, learning rate warmup) beyond those in standard implementations, since our goal was to compare architectures under a shared training protocol rather than to maximize each baseline individually.

\section{Related Work}
\label{app:related}

\textbf{Hopfield networks and associative memory.}
Modern (dense) Hopfield networks~\citep{krotov2016dense, demircigil2017model, ramsauer2020hopfield} store exponentially many patterns and connect directly to Transformer attention.
Extensions include hierarchical memory~\citep{krotov2021hierarchical}, the Energy Transformer~\citep{hoover2023energy}, and the Epanechnikov/LSR energy~\citep{hoover2025dense}.
\citet{liang2022modern} apply modern Hopfield retrieval to graph \emph{embedding} (graph-level classification), while we focus on \emph{node} classification with a joint energy that explicitly couples retrieval and graph smoothing.

\textbf{Graph transformers and content--structure coupling.}
Graph transformers such as Graphormer~\citep{ying2021graphormer} and subsequent scalable variants~\citep{rampasek2022recipe} combine content-based attention with structural biases.
Our goal is different: instead of replacing message passing with full transformer blocks, we formulate node updates as iterative descent steps on a joint energy, making the retrieval/smoothing tradeoff explicit and analyzable.

\textbf{Implicit and iterative graph models.}
Implicit/equilibrium GNNs solve fixed-point equations over graph representations~\citep{gu2020implicit}.
Our fixed-point view is related, but the operator is derived from a Hopfield-inspired retrieval energy plus Laplacian regularization, yielding explicit smoothness/contraction conditions for the base dynamics (Appendix~\ref{app:theory}).
Our iterative formulation also connects to analyses of depth-related degradation and expressivity loss in graph propagation~\citep{li2018deeper, oono2020graph}.

\textbf{Robust graph learning.}
GNN robustness is typically studied under adversarial structural perturbations~\citep{zugner2018adversarial}.
Our setting is complementary: we study random corruption to isolate the effect of structural reliability on model accuracy, without adversarial assumptions.
Our finding that aggregation design (sum vs.\ mean vs.\ teleportation) largely determines robustness connects to analyses of information flow bottlenecks in GNNs~\citep{topping2022understanding}.

\section{Additional Results}
\label{app:results}

\subsection{Edge Deletion on Planetoid}
\label{app:edge_robustness}

Table~\ref{tab:robustness} quantifies relative accuracy drop under edge deletion on the Planetoid datasets.
\ours{} robustness varies by dataset: LSE leads on PubMed (4.6\% drop), while on CiteSeer \ours{} variants are less robust (14--20\% drop) than GCN (10\%) and APPNP (6\%).
GIN is uniquely fragile under full edge removal (31--50\% loss) due to its injective sum aggregation~\citep{xu2019powerful}; GCN (mean) and APPNP (teleportation) degrade more gracefully (5--14\%).
GIN shows the opposite profile under feature masking: robust (91.1\% at 50\% masking) but fragile under edge drop, reflecting that injective aggregation exploits structure maximally while memory retrieval provides a structure-independent channel.

\begin{table}[h]
\centering
\caption{Relative accuracy drop (\%) under edge deletion on Planetoid. Smaller magnitude = greater robustness. \textbf{Bold}: most robust per dataset.}
\label{tab:robustness}
\small
\begin{tabular}{@{}lcccccc@{}}
\toprule
& \multicolumn{3}{c}{50\% edge drop} & \multicolumn{3}{c}{100\% edge drop} \\
\cmidrule(lr){2-4} \cmidrule(lr){5-7}
& Cora & CiteSeer & PubMed & Cora & CiteSeer & PubMed \\
\midrule
\ours{}-LSE      & $-6.5$ & $-10.2$  & $\mathbf{-1.8}$ & $-15.5$ & $-19.5$ & $\mathbf{-4.6}$ \\
\ours{}-Hier(8)  & $\mathbf{-5.6}$ & $-8.1$ & $-2.5$          & $-13.5$          & $-14.2$ & $-6.1$ \\
GCN              & $-5.9$ & $-5.5$ & $-2.2$   & $-13.6$          & $-9.7$ & $-4.8$ \\
APPNP            & $-6.1$ & $\mathbf{-3.0}$ & $-2.3$ & $\mathbf{-12.5}$ & $\mathbf{-6.3}$ & $-5.5$ \\
GIN              & $-12.7$ & $-16.5$ & $-10.3$           & $-44.1$          & $-49.9$ & $-30.5$ \\
\bottomrule
\end{tabular}
\end{table}

\textbf{Feature masking across datasets.}
Figure~\ref{fig:feature_mask_all} shows feature masking degradation across all five datasets.
Unlike edge deletion, all models (including \ours{}) degrade at comparable rates as features are masked on these datasets.
However, on the Amazon datasets with $K{=}256$ and tuned hyperparameters (Table~\ref{tab:amazon_corruption}), memory-enabled variants show substantially greater robustness to feature masking than NoMem (e.g., Hier: 91.9\% vs.\ NoMem: 86.9\% at 50\% masking on Photo), suggesting that higher pattern capacity provides redundant representations that compensate for missing features, a benefit not observed at $K{=}64$ on Planetoid.

\begin{figure}[h]
\centering
\includegraphics[width=\linewidth]{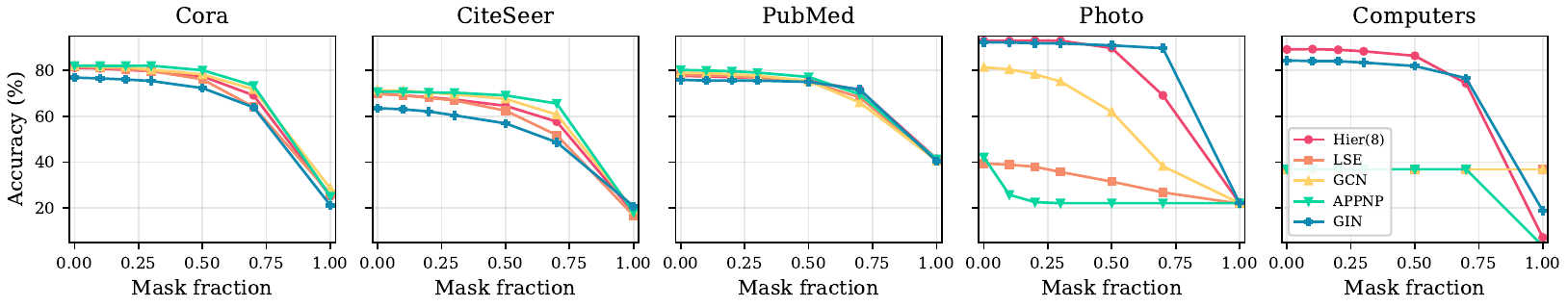}
\caption{Accuracy under feature masking across five datasets. Tuning selects $K{=}64$ on Planetoid and $K{=}256$ on Amazon. All models degrade at comparable rates on Planetoid; on Amazon, memory-enabled variants show substantially greater robustness (Table~\ref{tab:amazon_corruption}).}
\label{fig:feature_mask_all}
\end{figure}

\textbf{LSE vs LSR phase diagram.}
Figure~\ref{fig:phase_diagram} shows the $\beta \times K$ phase diagram on Amazon Photo.
LSE (left) is stable only at $K{=}256$; at $K{\leq}128$ it exhibits bimodal collapse where some seeds converge and others do not, regardless of $\beta$.
LSR (right) achieves stable ${\geq}93\%$ accuracy at $K{=}64$ for $\beta \in \{0.2, 0.5, 1.0, 5.0\}$, but is also bimodal at very low $\beta$ ($<0.2$) and at $\beta{=}2.0$.

\begin{figure}[h]
\centering
\includegraphics[width=\linewidth]{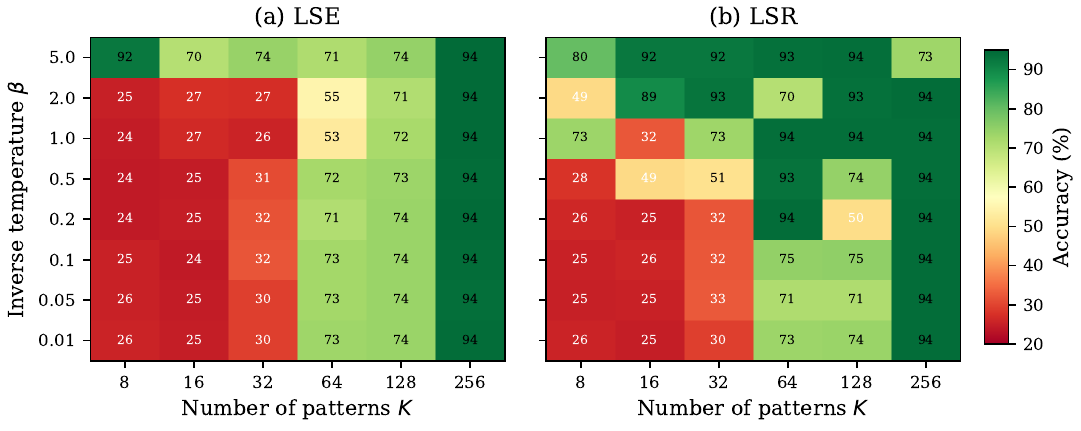}
\caption{Phase diagram of LSE vs.\ LSR energy on Amazon Photo. Each cell shows mean test accuracy over 10 seeds. High variance in some cells reflects bimodal behavior (seeds either converge to ${\sim}94\%$ or collapse). At $K{=}64$, LSR is stable for moderate $\beta$ while LSE is bimodal for all tested $\beta$.}
\label{fig:phase_diagram}
\end{figure}

\textbf{Memory contribution on Amazon.}
Under per-model hyperparameter tuning (Table~\ref{tab:accuracy}), NoMem slightly outperforms all memory-enabled variants on both Amazon datasets (94.7\% vs.\ 94.5\% on Photo; 91.5\% vs.\ 91.3\% on Computers).
This contrasts with the consistent memory benefit on Planetoid and suggests that the Amazon co-purchase graphs have richer edge structure (Photo: 238k edges vs.\ Cora: 10.6k edges) making Laplacian smoothing alone sufficient.
The \ours{} architecture's advantage on Amazon is its iterative Laplacian propagation, not memory retrieval per se; all \ours{} variants outperform the best standard baseline by 0.8--2.6~pp.

\subsection{Laplacian Weight $\lambda$ Ablation}
\label{app:lambda_ablation}

Table~\ref{tab:lambda_ablation} shows GHN-LSE accuracy across $\lambda \in \{0, 0.01, 0.05, 0.1, 0.2, 0.3, 0.5, 1.0\}$ on all seven datasets.
On homophilous datasets (Cora, CiteSeer, PubMed), accuracy increases monotonically from $\lambda{=}0$ to a peak at $\lambda{=}0.2$--$0.3$, then degrades mildly.
On heterophilous datasets (Texas, Wisconsin, Cornell, Actor), accuracy decreases monotonically with increasing $\lambda$, confirming that Laplacian smoothing is harmful when neighbors have different labels.

\begin{table}[h]
\centering
\caption{GHN-LSE test accuracy (\%) as a function of $\lambda$. \textbf{Bold}: best $\lambda$ per dataset.}
\label{tab:lambda_ablation}
\small
\begin{tabular}{@{}lcccccccc@{}}
\toprule
Dataset & $0$ & $0.01$ & $0.05$ & $0.1$ & $0.2$ & $0.3$ & $0.5$ & $1.0$ \\
\midrule
Cora      & 57.7 & 62.7 & 71.7 & 76.1 & 80.2 & \textbf{81.8} & 81.6 & 80.6 \\
CiteSeer  & 55.6 & 59.1 & 65.1 & 67.6 & \textbf{70.2} & 69.6 & 68.9 & 69.4 \\
PubMed    & 73.0 & 73.2 & 74.5 & 75.2 & \textbf{78.2} & 77.7 & 76.9 & 76.7 \\
\midrule
Texas     & \textbf{74.1} & 71.6 & 70.0 & 64.9 & 66.8 & 67.0 & 66.8 & 64.9 \\
Wisconsin & 79.8 & \textbf{80.8} & 72.2 & 67.1 & 60.0 & 55.1 & 45.3 & 44.9 \\
Cornell   & \textbf{73.5} & 72.4 & 63.8 & 46.2 & 41.1 & 44.3 & 42.2 & 41.4 \\
Actor     & \textbf{37.5} & 37.3 & 36.5 & 35.5 & 30.8 & 28.1 & 25.5 & 25.0 \\
\bottomrule
\end{tabular}
\end{table}

The results for other GHN variants (LSR, Hier, NoMem) follow the same pattern.
On Cora, GHN-LSR peaks at $\lambda{=}0.3$ (81.5\%) and GHN-Hier(8) at $\lambda{=}0.5$ (81.7\%).
On heterophilous datasets, all variants achieve their best performance at $\lambda{=}0$ or $\lambda{=}0.01$.

\subsection{Negative $\lambda$ (Graph Sharpening)}
\label{app:negative_lambda}

Setting $\lambda < 0$ reverses the Laplacian regularization: instead of pulling neighbors together, it pushes them apart (``graph sharpening'').
Table~\ref{tab:neg_lambda} shows that small negative values ($\lambda \in \{-0.05, -0.01\}$) can improve upon $\lambda{=}0$ on heterophilous datasets.

\begin{table}[h]
\centering
\caption{GHN accuracy (\%) with negative $\lambda$ on heterophilous datasets. Best result per dataset-model in \textbf{bold}.}
\label{tab:neg_lambda}
\small
\begin{tabular}{@{}llcccc@{}}
\toprule
& & $\lambda{=}{-}0.3$ & $\lambda{=}{-}0.1$ & $\lambda{=}{-}0.05$ & $\lambda{=}{-}0.01$ \\
\midrule
\multirow{2}{*}{Texas}
& LSE & 75.9 & 77.3 & \textbf{78.6} & 74.6 \\
& LSR & 75.9 & 77.3 & \textbf{79.5} & 78.6 \\
\midrule
\multirow{2}{*}{Wisconsin}
& LSE & 76.9 & 81.0 & \textbf{81.8} & 79.0 \\
& LSR & 78.8 & \textbf{80.8} & 80.4 & 79.0 \\
\midrule
\multirow{2}{*}{Cornell}
& LSE & 61.9 & 66.5 & 74.6 & \textbf{77.0} \\
& LSR & 63.2 & 69.5 & 74.3 & \textbf{77.8} \\
\midrule
\multirow{2}{*}{Actor}
& LSE & 34.7 & 37.5 & \textbf{38.0} & 37.5 \\
& LSR & 34.9 & 37.2 & 37.3 & \textbf{37.8} \\
\bottomrule
\end{tabular}
\end{table}

The optimal negative $\lambda$ is dataset-dependent: $-0.05$ works best on Texas and Wisconsin, while $-0.01$ is preferred on Cornell.
Gains over $\lambda{=}0$ are modest but consistent (0.5--3.5~pp).
Note that the accuracies here differ from Table~\ref{tab:heterophilous} because this ablation varies $\lambda$ with other hyperparameters (hidden dimension, dropout, $K$) held at a single default configuration, whereas Table~\ref{tab:heterophilous} reports per-model HP tuning across the full grid.
Very negative values ($\lambda{=}{-}0.3$) degrade performance, suggesting that excessive sharpening is as harmful as excessive smoothing.

\subsection{Iteration Count $T$ Ablation}
\label{app:t_ablation}

Table~\ref{tab:t_ablation} shows how the number of Hopfield update iterations $T$ per layer affects accuracy on one homophilous (Cora) and two heterophilous (Texas, Wisconsin) datasets, using $\lambda{=}0.3$.

\begin{table}[h]
\centering
\caption{GHN-LSE accuracy (\%) vs.\ iterations $T$ with $\lambda{=}0.3$. \textbf{Bold}: best $T$.}
\label{tab:t_ablation}
\small
\begin{tabular}{@{}lccccc@{}}
\toprule
& $T{=}1$ & $T{=}2$ & $T{=}4$ & $T{=}8$ & $T{=}16$ \\
\midrule
Cora      & $74.5 {\scriptstyle\pm 1.1}$ & $78.4 {\scriptstyle\pm 0.5}$ & $\mathbf{81.8 {\scriptstyle\pm 1.4}}$ & $56.1 {\scriptstyle\pm 26.1}$ & $30.9 {\scriptstyle\pm 2.0}$ \\
Texas     & $\mathbf{73.5 {\scriptstyle\pm 2.8}}$ & $72.2 {\scriptstyle\pm 2.9}$ & $67.0 {\scriptstyle\pm 1.7}$ & $65.9 {\scriptstyle\pm 1.4}$ & $65.4 {\scriptstyle\pm 1.1}$ \\
Wisconsin & $\mathbf{63.7 {\scriptstyle\pm 1.0}}$ & $58.6 {\scriptstyle\pm 2.3}$ & $55.1 {\scriptstyle\pm 5.9}$ & $44.1 {\scriptstyle\pm 1.9}$ & $46.3 {\scriptstyle\pm 1.4}$ \\
\bottomrule
\end{tabular}
\end{table}

On Cora, LSE peaks at $T{=}4$ and collapses at $T{\geq}8$ (the large standard deviation at $T{=}8$ indicates bimodal convergence).
On heterophilous datasets, $T{=}1$ is best; additional iterations with $\lambda{=}0.3$ apply progressively more harmful smoothing.
This is consistent with the $\lambda$ ablation results: on heterophilous graphs, minimizing the influence of the Laplacian term (whether by setting $\lambda{=}0$ or $T{=}1$) is essential.

\subsection{Memory Heads $H$ Ablation}
\label{app:heads_ablation}

Table~\ref{tab:heads_ablation} shows GHN-LSE accuracy with $H \in \{1,2,4,8\}$ memory heads on four datasets.

\begin{table}[h]
\centering
\caption{GHN-LSE accuracy (\%) vs.\ number of memory heads $H$.}
\label{tab:heads_ablation}
\small
\begin{tabular}{@{}lcccc@{}}
\toprule
& $H{=}1$ & $H{=}2$ & $H{=}4$ & $H{=}8$ \\
\midrule
Cora     & $81.5 {\scriptstyle\pm 0.9}$ & $81.6 {\scriptstyle\pm 1.0}$ & $81.8 {\scriptstyle\pm 1.4}$ & $80.3 {\scriptstyle\pm 0.4}$ \\
CiteSeer & $69.6 {\scriptstyle\pm 1.4}$ & $70.0 {\scriptstyle\pm 1.6}$ & $69.6 {\scriptstyle\pm 0.7}$ & $69.5 {\scriptstyle\pm 1.5}$ \\
PubMed   & $77.6 {\scriptstyle\pm 0.5}$ & $77.6 {\scriptstyle\pm 0.6}$ & $77.7 {\scriptstyle\pm 0.9}$ & $77.9 {\scriptstyle\pm 1.0}$ \\
Texas    & $64.9 {\scriptstyle\pm 1.8}$ & $65.9 {\scriptstyle\pm 2.9}$ & $67.0 {\scriptstyle\pm 1.7}$ & $67.6 {\scriptstyle\pm 0.0}$ \\
\bottomrule
\end{tabular}
\end{table}

The number of memory heads has minimal effect on accuracy: across all four datasets, the difference between the best and worst $H$ is at most 1.5~pp on Cora and 2.7~pp on Texas.
This supports using $H{=}1$ as the default for simplicity without sacrificing performance.

\subsection{Gate Analysis under Feature Corruption}
\label{app:gate_analysis}

We train GHN-LSE on Amazon Photo (300 epochs) and extract the learned gate value $g_v = \sigma(W_g[\mathbf{x}_v \| \hat{\mathbf{x}}_v] + b_g)$ under feature masking at levels 0\%, 10\%, 30\%, 50\%, and 70\%.
Table~\ref{tab:gate_analysis} reports the mean gate value (averaged over nodes and iterations) at each corruption level.

\begin{table}[h]
\centering
\caption{Mean gate value $\bar{g}$ of GHN-LSE on Amazon Photo under increasing feature masking. Higher $g$ means more weight on memory retrieval.}
\label{tab:gate_analysis}
\small
\begin{tabular}{@{}lccccc@{}}
\toprule
Mask \% & 0\% & 10\% & 30\% & 50\% & 70\% \\
\midrule
$\bar{g}$ & $0.727 {\scriptstyle\pm 0.039}$ & $0.727 {\scriptstyle\pm 0.039}$ & $0.726 {\scriptstyle\pm 0.038}$ & $0.725 {\scriptstyle\pm 0.038}$ & $0.723 {\scriptstyle\pm 0.037}$ \\
Accuracy (\%) & $93.6 {\scriptstyle\pm 0.5}$ & $93.6 {\scriptstyle\pm 0.5}$ & $93.2 {\scriptstyle\pm 0.5}$ & $90.6 {\scriptstyle\pm 1.3}$ & $67.9 {\scriptstyle\pm 12.8}$ \\
\bottomrule
\end{tabular}
\end{table}

The gate remains nearly constant ($\Delta\bar{g} \leq 0.004$) even as accuracy degrades substantially under heavy corruption.
This indicates the gate does not dynamically adapt its memory/feature blending ratio at inference time; instead, robustness arises because the stored prototype patterns provide useful class-level information even when node features are degraded.

\section{Theoretical Analysis for Base LSE Dynamics}
\label{app:theory}

The standard Hopfield energy is guaranteed to decrease under asynchronous updates of a single query~\citep{ramsauer2020hopfield}.
Our coupled scheme applies synchronous updates across all nodes with the Laplacian term introducing inter-node coupling, so this guarantee does not directly apply.
Below we provide formal results for the \emph{base} GHN energy with LSE retrieval and fixed $(\mathbf{M},\beta,\lambda,\mathbf{L})$ during iterative updates: explicit smoothness and contraction conditions for gradient-descent and fixed-point updates.
These results do not cover the full training-time system with gating, hierarchical routing, and jointly learned parameters; for the full model we rely on empirical stability (stable convergence within $T{=}4$ iterations in all experiments).
The theory assumes $\lambda \ge 0$; for $\lambda < 0$ (graph sharpening on heterophilous benchmarks), the Laplacian term becomes indefinite and our convexity/contraction guarantees no longer apply.
We keep $|\lambda|$ small ($\lambda \in \{-0.1, -0.05, -0.01\}$) and rely on damping, early stopping, and empirical stability in those experiments.
Table~\ref{tab:beta_norm} reports $\beta\|\mathbf{M}\|_\sigma^2$ (product of learned $\beta$ and spectral norm squared of $\mathbf{M}$) at trained operating points, averaged over 5 seeds.
The convexity condition $\beta\|\mathbf{M}\|^2 < 2$ is satisfied only for Cornell (0.75) and CiteSeer (1.88); all other datasets exceed the threshold, with Amazon Photo reaching 35.96.
This places the system well outside the strongly convex regime of Proposition~\ref{prop:convexity_base} in typical use; gradient descent nonetheless converges reliably within $T{=}4$ iterations in all experiments, suggesting tighter analysis (e.g., local convexity near attractors or Łojasiewicz-type arguments) may close this gap.

\begin{table}[h]
\centering
\caption{Learned $\beta\|\mathbf{M}\|_\sigma^2$ at trained operating points (mean over 5 seeds). The convexity condition requires this product to be ${<}\,2$.}
\label{tab:beta_norm}
\small
\begin{tabular}{@{}lccc@{}}
\toprule
Dataset & $\beta$ & $\|\mathbf{M}\|_\sigma^2$ & $\beta\|\mathbf{M}\|_\sigma^2$ \\
\midrule
Cora         & 1.11 & 1.94  & 2.15 \\
CiteSeer     & 1.09 & 1.71  & \textbf{1.88} \\
PubMed       & 1.36 & 5.24  & 6.54 \\
Amazon Photo & 0.99 & 36.37 & 35.96 \\
Amazon Comp. & 1.12 & 2.11  & 2.36 \\
Texas        & 1.25 & 3.88  & 4.65 \\
Wisconsin    & 1.05 & 2.48  & 2.60 \\
Cornell      & 1.07 & 0.68  & \textbf{0.75} \\
Actor        & 1.00 & 4.61  & 4.72 \\
\bottomrule
\end{tabular}
\end{table}

\subsection{Setup}
\label{app:theory_setup}

Let $\mathbf{X}\in\mathbb{R}^{N\times d}$ be the node representation matrix with rows $\mathbf{x}_v$, and let $\mathbf{M}\in\mathbb{R}^{K\times d}$ be the memory matrix.
For fixed $\beta>0$ and $\lambda\ge 0$, define
\begin{equation}
\label{eq:app_energy_base}
E_{\mathrm{base}}(\mathbf{X})
=
\sum_{v=1}^{N}
\left[
-\beta^{-1}\log\!\sum_{\mu=1}^{K}\exp\!\left(\beta\,\mathbf{m}_\mu^\top\mathbf{x}_v\right)
+\frac12\|\mathbf{x}_v\|^2
\right]
+\lambda\,\tr(\mathbf{X}^\top \mathbf{L}\mathbf{X}),
\end{equation}
where $\mathbf{L}=\mathbf{I}-\mathbf{D}^{-1/2}\mathbf{A}\mathbf{D}^{-1/2}$ is the symmetric normalized Laplacian.

\begin{proposition}[Normalized Laplacian quadratic form]
\label{prop:normalized_laplacian_identity}
If $\mathbf{A}$ is symmetric and $\mathbf{D}_{vv}=d_v>0$, then
\[
\tr(\mathbf{X}^\top \mathbf{L}\mathbf{X})
=
\frac12\sum_{u,v}A_{uv}\left\|
\frac{\mathbf{x}_u}{\sqrt{d_u}}
-\frac{\mathbf{x}_v}{\sqrt{d_v}}
\right\|^2
\ge 0.
\]
Hence $\mathbf{L}\succeq 0$ and the Laplacian term is convex.
\end{proposition}

\begin{proof}
Define $\mathbf{Y}:=\mathbf{D}^{-1/2}\mathbf{X}$, so $\mathbf{y}_v=\mathbf{x}_v/\sqrt{d_v}$.
Using $\mathbf{L}=\mathbf{I}-\mathbf{D}^{-1/2}\mathbf{A}\mathbf{D}^{-1/2}$,
\[
\tr(\mathbf{X}^\top\mathbf{L}\mathbf{X})
=
\tr(\mathbf{X}^\top\mathbf{X})-\tr(\mathbf{Y}^\top\mathbf{A}\mathbf{Y})
=
\sum_v\|\mathbf{x}_v\|^2-\sum_{u,v}A_{uv}\,\mathbf{y}_u^\top\mathbf{y}_v.
\]
Now expand
\begin{align}
\frac12\sum_{u,v}A_{uv}\|\mathbf{y}_u-\mathbf{y}_v\|^2
&=
\frac12\sum_{u,v}A_{uv}\left(\|\mathbf{y}_u\|^2+\|\mathbf{y}_v\|^2-2\mathbf{y}_u^\top\mathbf{y}_v\right) \\
&=
\sum_u\left(\sum_v A_{uv}\right)\|\mathbf{y}_u\|^2-\sum_{u,v}A_{uv}\mathbf{y}_u^\top\mathbf{y}_v \\
&=
\sum_u d_u\left\|\frac{\mathbf{x}_u}{\sqrt{d_u}}\right\|^2-\sum_{u,v}A_{uv}\frac{\mathbf{x}_u^\top\mathbf{x}_v}{\sqrt{d_ud_v}} \\
&=
\sum_u\|\mathbf{x}_u\|^2-\sum_{u,v}A_{uv}\frac{\mathbf{x}_u^\top\mathbf{x}_v}{\sqrt{d_ud_v}} \\
&=
\tr(\mathbf{X}^\top\mathbf{L}\mathbf{X}).
\end{align}
This proves the identity. The right-hand side is a sum of squared norms with nonnegative weights, so it is nonnegative.
\end{proof}

\subsection{Gradient and Hessian Structure}
\label{app:theory_diff}

\begin{proposition}[Gradient formula]
\label{prop:gradient_formula}
For each node $v$,
\begin{equation}
\label{eq:app_gradient_node}
\nabla_{\mathbf{x}_v}E_{\mathrm{base}}(\mathbf{X})
=
-\mathbf{M}^\top\softmax(\beta\,\mathbf{M}\mathbf{x}_v)
+\mathbf{x}_v
+2\lambda(\mathbf{L}\mathbf{X})_v.
\end{equation}
Therefore the damped update in Eq.~\ref{eq:update} is exactly gradient descent on $E_{\mathrm{base}}$ with step size $\alpha$.
\end{proposition}

\begin{proof}
Write
\[
E_{\mathrm{base}}(\mathbf{X})=
\sum_{u=1}^N e_u(\mathbf{x}_u)+\lambda\,\tr(\mathbf{X}^\top\mathbf{L}\mathbf{X}),
\quad
e_u(\mathbf{x})=-\beta^{-1}\log\sum_{\mu=1}^K e^{\beta\mathbf{m}_\mu^\top\mathbf{x}}+\frac12\|\mathbf{x}\|^2.
\]
Only $e_v(\mathbf{x}_v)$ and the Laplacian term depend on $\mathbf{x}_v$.
For retrieval, set $\mathbf{z}=\mathbf{M}\mathbf{x}_v$ and use
$\nabla_{\mathbf{z}}\left[\beta^{-1}\log\sum_\mu e^{\beta z_\mu}\right]=\softmax(\beta\mathbf{z})$.
Chain rule gives
$\nabla_{\mathbf{x}_v}\lse(\beta,\mathbf{M}\mathbf{x}_v)=\mathbf{M}^\top\softmax(\beta\mathbf{M}\mathbf{x}_v)$.
The quadratic term contributes $\mathbf{x}_v$.
For the Laplacian term,
\[
\nabla_{\mathbf{X}}\tr(\mathbf{X}^\top\mathbf{L}\mathbf{X})=(\mathbf{L}+\mathbf{L}^\top)\mathbf{X}=2\mathbf{L}\mathbf{X},
\]
since $\mathbf{L}$ is symmetric.
Combining terms yields Eq.~\ref{eq:app_gradient_node}.
Then
\[
\mathbf{x}_v^{(t+1)}
=
\mathbf{x}_v^{(t)}-\alpha\nabla_{\mathbf{x}_v}E_{\mathrm{base}}(\mathbf{X}^{(t)})
=
(1-\alpha)\mathbf{x}_v^{(t)}
+\alpha\!\left[\mathbf{M}^\top\softmax(\beta\mathbf{M}\mathbf{x}_v^{(t)})-2\lambda(\mathbf{L}\mathbf{X}^{(t)})_v\right].
\]
\end{proof}

\begin{lemma}[Retrieval Jacobian]
\label{lem:retrieval_jacobian}
Let $s(\mathbf{x})=\mathbf{M}^\top\softmax(\beta\mathbf{M}\mathbf{x})$ and $\mathbf{p}(\mathbf{x})=\softmax(\beta\mathbf{M}\mathbf{x})$.
Then
\[
\nabla s(\mathbf{x})=\beta\,\mathbf{M}^\top\Sigma(\mathbf{p}(\mathbf{x}))\mathbf{M},
\quad
\Sigma(\mathbf{p})=\diag(\mathbf{p})-\mathbf{p}\mathbf{p}^\top.
\]
\end{lemma}

\begin{proof}
For softmax $\mathbf{p}$, componentwise Jacobian is
\[
\frac{\partial p_i}{\partial z_j}=p_i(\delta_{ij}-p_j),
\]
so in matrix form $\nabla_{\mathbf{z}}\softmax(\mathbf{z})=\diag(\mathbf{p})-\mathbf{p}\mathbf{p}^\top$.
With $\mathbf{z}=\beta\mathbf{M}\mathbf{x}$,
\[
\nabla s(\mathbf{x})
=
\mathbf{M}^\top
\left(\diag(\mathbf{p})-\mathbf{p}\mathbf{p}^\top\right)
(\beta\mathbf{M})
=
\beta\,\mathbf{M}^\top\Sigma(\mathbf{p})\mathbf{M}.
\]
\end{proof}

\begin{lemma}[Softmax covariance bound]
\label{lem:softmax_cov}
For any probability vector $\mathbf{p}\in\Delta^{K-1}$, define
$\Sigma(\mathbf{p})=\diag(\mathbf{p})-\mathbf{p}\mathbf{p}^\top$.
Then $\Sigma(\mathbf{p})\succeq 0$ and $\|\Sigma(\mathbf{p})\|\le \frac12$.
\end{lemma}

\begin{proof}
$\Sigma(\mathbf{p})$ is the covariance matrix of a categorical one-hot variable, hence PSD.
For unit $\mathbf{u}$,
\[
\mathbf{u}^\top\Sigma(\mathbf{p})\mathbf{u}
=\sum_i p_i u_i^2-\left(\sum_i p_i u_i\right)^2
=\mathrm{Var}(u_I)
\le \frac{(\max_i u_i-\min_i u_i)^2}{4}
\le \frac12,
\]
by Popoviciu's inequality and $\max_i u_i-\min_i u_i\le \sqrt{2}\|\mathbf{u}\|=\sqrt{2}$.
Taking the supremum over unit $\mathbf{u}$ gives $\|\Sigma(\mathbf{p})\|\le \frac12$.
\end{proof}

\begin{proposition}[Convexity and strong convexity regime]
\label{prop:convexity_base}
Let
$\mu:=1-\frac{\beta\|\mathbf{M}\|^2}{2}$.
Then the Hessian satisfies
\[
\nabla^2 E_{\mathrm{base}}(\mathbf{X})
\succeq
\mu\,\mathbf{I}_{Nd}
+2\lambda(\mathbf{L}\otimes \mathbf{I}_d).
\]
Consequently:
\begin{enumerate}
\item if $\beta\|\mathbf{M}\|^2\le 2$, $E_{\mathrm{base}}$ is convex;
\item if $\beta\|\mathbf{M}\|^2<2$, $E_{\mathrm{base}}$ is $\mu$-strongly convex.
\end{enumerate}
\end{proposition}

\begin{proof}
For each node,
\[
\nabla_{\mathbf{x}_v}^2\lse(\beta,\mathbf{M}\mathbf{x}_v)
=
\beta\,\mathbf{M}^\top\Sigma(\mathbf{p}_v)\mathbf{M},
\quad
\mathbf{p}_v:=\softmax(\beta\mathbf{M}\mathbf{x}_v).
\]
By Lemma~\ref{lem:softmax_cov},
\[
\mathbf{0}
\preceq
\beta\mathbf{M}^\top\Sigma(\mathbf{p}_v)\mathbf{M}
\preceq
\frac{\beta}{2}\|\mathbf{M}\|^2\mathbf{I}_d.
\]
Hence
\[
\nabla_{\mathbf{x}_v}^2\!\left[-\lse(\beta,\mathbf{M}\mathbf{x}_v)+\frac12\|\mathbf{x}_v\|^2\right]
=
\mathbf{I}_d-\beta\mathbf{M}^\top\Sigma(\mathbf{p}_v)\mathbf{M}
\succeq
\left(1-\frac{\beta\|\mathbf{M}\|^2}{2}\right)\mathbf{I}_d.
\]
Stacking all nodes and adding the Laplacian Hessian contribution $2\lambda(\mathbf{L}\otimes\mathbf{I}_d)\succeq 0$ yields
\[
\nabla^2E_{\mathrm{base}}(\mathbf{X})
\succeq
\left(1-\frac{\beta\|\mathbf{M}\|^2}{2}\right)\mathbf{I}_{Nd}
+2\lambda(\mathbf{L}\otimes\mathbf{I}_d).
\]
This implies convexity when $\beta\|\mathbf{M}\|^2\le 2$, and $\mu$-strong convexity when $\beta\|\mathbf{M}\|^2<2$.
\end{proof}

\subsection{Smoothness and Existence of Minimizers}
\label{app:theory_smooth}

\begin{proposition}[$L$-smoothness]
\label{prop:smoothness_base}
$\nabla E_{\mathrm{base}}$ is Lipschitz in Frobenius norm:
\[
\|\nabla E_{\mathrm{base}}(\mathbf{X})-\nabla E_{\mathrm{base}}(\mathbf{Y})\|_F
\le
L_{\mathrm{lip}}\|\mathbf{X}-\mathbf{Y}\|_F,
\]
with
\[
L_{\mathrm{lip}}=\frac{\beta}{2}\|\mathbf{M}\|^2+1+2\lambda\|\mathbf{L}\|.
\]
\end{proposition}

\begin{proof}
Define row-wise retrieval
$s(\mathbf{x})=\mathbf{M}^\top\softmax(\beta\mathbf{M}\mathbf{x})$ and $[S(\mathbf{X})]_v=s(\mathbf{x}_v)$.
From Proposition~\ref{prop:gradient_formula},
\[
\nabla E_{\mathrm{base}}(\mathbf{X})=\mathbf{X}-S(\mathbf{X})+2\lambda\mathbf{L}\mathbf{X}.
\]
Hence
\begin{align}
\|\nabla E(\mathbf{X})-\nabla E(\mathbf{Y})\|_F
&\le
\|\mathbf{X}-\mathbf{Y}\|_F
+\|S(\mathbf{X})-S(\mathbf{Y})\|_F
+2\lambda\|\mathbf{L}\|\,\|\mathbf{X}-\mathbf{Y}\|_F \\
&=
\left(1+2\lambda\|\mathbf{L}\|\right)\|\mathbf{X}-\mathbf{Y}\|_F
+\|S(\mathbf{X})-S(\mathbf{Y})\|_F.
\label{eq:lip_reduce_S}
\end{align}
By Lemma~\ref{lem:retrieval_jacobian} and Lemma~\ref{lem:softmax_cov},
\[
\|\nabla s(\mathbf{x})\|
=
\|\beta\mathbf{M}^\top\Sigma(\mathbf{p})\mathbf{M}\|
\le
\beta\|\mathbf{M}\|^2\|\Sigma(\mathbf{p})\|
\le
\frac{\beta}{2}\|\mathbf{M}\|^2.
\]
Thus $s$ is $\frac{\beta}{2}\|\mathbf{M}\|^2$-Lipschitz by the mean value theorem:
\[
\|s(\mathbf{x})-s(\mathbf{y})\|
\le
\frac{\beta}{2}\|\mathbf{M}\|^2\|\mathbf{x}-\mathbf{y}\|.
\]
Apply this row-wise:
\[
\|S(\mathbf{X})-S(\mathbf{Y})\|_F^2
\le
\left(\frac{\beta}{2}\|\mathbf{M}\|^2\right)^2\|\mathbf{X}-\mathbf{Y}\|_F^2,
\]
so
\[
\|S(\mathbf{X})-S(\mathbf{Y})\|_F
\le
\frac{\beta}{2}\|\mathbf{M}\|^2\|\mathbf{X}-\mathbf{Y}\|_F.
\]
Substitute into Eq.~\ref{eq:lip_reduce_S} to obtain the claimed constant.
\end{proof}

\begin{corollary}[Step-size stability condition]
\label{cor:step_size_base}
For gradient descent
$\mathbf{X}^{(t+1)}=\mathbf{X}^{(t)}-\eta\nabla E_{\mathrm{base}}(\mathbf{X}^{(t)})$,
any $\eta\in(0,2/L_{\mathrm{lip}})$ guarantees monotone descent.
Since $\|\mathbf{L}\|\le 2$ for the normalized Laplacian, a convenient sufficient bound is
\[
\eta
<
\frac{2}{\frac{\beta}{2}\|\mathbf{M}\|^2+1+4\lambda}.
\]
\end{corollary}

\begin{proof}
The first claim is the standard smooth-gradient step-size condition.
The second follows from Proposition~\ref{prop:smoothness_base} and $\|\mathbf{L}\|\le 2$.
\end{proof}

\begin{proposition}[Coercivity and existence of a minimizer]
\label{prop:coercive_base}
$E_{\mathrm{base}}(\mathbf{X})\to +\infty$ as $\|\mathbf{X}\|_F\to\infty$.
Therefore $E_{\mathrm{base}}$ attains at least one global minimizer.
\end{proposition}

\begin{proof}
For each node, with $\mathbf{z}=\mathbf{M}\mathbf{x}_v$,
\[
\lse(\beta,\mathbf{M}\mathbf{x}_v)
=\beta^{-1}\log\sum_{\mu=1}^K e^{\beta z_\mu}
\le \beta^{-1}\left(\log K + \beta\max_\mu z_\mu\right)
\le \beta^{-1}\log K+\|\mathbf{M}\|\,\|\mathbf{x}_v\|.
\]
Hence
\[
E_{\mathrm{base}}(\mathbf{X})
\ge
\frac12\|\mathbf{X}\|_F^2
-\|\mathbf{M}\|\sum_v\|\mathbf{x}_v\|
-N\beta^{-1}\log K
+\lambda\tr(\mathbf{X}^\top\mathbf{L}\mathbf{X}).
\]
Using $\sum_v\|\mathbf{x}_v\|\le \sqrt{N}\|\mathbf{X}\|_F$ and
$\tr(\mathbf{X}^\top\mathbf{L}\mathbf{X})\ge 0$ from Proposition~\ref{prop:normalized_laplacian_identity},
\[
E_{\mathrm{base}}(\mathbf{X})
\ge
\frac12\|\mathbf{X}\|_F^2
-\|\mathbf{M}\|\sqrt{N}\,\|\mathbf{X}\|_F
-N\beta^{-1}\log K.
\]
The right-hand side is a coercive quadratic lower bound in $\|\mathbf{X}\|_F$, so it diverges to $+\infty$ as $\|\mathbf{X}\|_F\to\infty$.
Thus $E_{\mathrm{base}}$ is coercive. By continuity, it attains a global minimizer.
\end{proof}

\subsection{Gradient-Descent Convergence}
\label{app:theory_gd}

\begin{theorem}[Monotone descent and convergence to stationarity]
\label{thm:gd_base}
Assume $\eta\in(0,2/L_{\mathrm{lip}})$ and iterate
\[
\mathbf{X}^{(t+1)}=\mathbf{X}^{(t)}-\eta\nabla E_{\mathrm{base}}(\mathbf{X}^{(t)}).
\]
Then:
\begin{enumerate}
\item $E_{\mathrm{base}}(\mathbf{X}^{(t)})$ is non-increasing and
\[
E_{\mathrm{base}}(\mathbf{X}^{(t+1)})
\le
E_{\mathrm{base}}(\mathbf{X}^{(t)})
-\eta\left(1-\frac{\eta L_{\mathrm{lip}}}{2}\right)\|\nabla E_{\mathrm{base}}(\mathbf{X}^{(t)})\|_F^2;
\]
\item $\sum_{t=0}^\infty \|\nabla E_{\mathrm{base}}(\mathbf{X}^{(t)})\|_F^2<\infty$, so $\|\nabla E_{\mathrm{base}}(\mathbf{X}^{(t)})\|_F\to 0$;
\item every accumulation point is stationary.
\end{enumerate}
\end{theorem}

\begin{proof}
By Proposition~\ref{prop:smoothness_base}, $E_{\mathrm{base}}$ is $L_{\mathrm{lip}}$-smooth, so for any $\mathbf{X},\mathbf{Y}$,
\begin{equation}
\label{eq:descent_lemma_base}
E(\mathbf{Y})\le E(\mathbf{X})+\langle \nabla E(\mathbf{X}),\mathbf{Y}-\mathbf{X}\rangle+\frac{L_{\mathrm{lip}}}{2}\|\mathbf{Y}-\mathbf{X}\|_F^2.
\end{equation}
Set $\mathbf{Y}=\mathbf{X}-\eta\nabla E(\mathbf{X})$:
\begin{align}
E(\mathbf{X}-\eta\nabla E(\mathbf{X}))
&\le
E(\mathbf{X})
-\eta\|\nabla E(\mathbf{X})\|_F^2
+\frac{L_{\mathrm{lip}}\eta^2}{2}\|\nabla E(\mathbf{X})\|_F^2 \\
&=
E(\mathbf{X})
-\eta\left(1-\frac{\eta L_{\mathrm{lip}}}{2}\right)\|\nabla E(\mathbf{X})\|_F^2.
\end{align}
Applying this at $\mathbf{X}=\mathbf{X}^{(t)}$ proves item (1).
Let $c:=\eta(1-\eta L_{\mathrm{lip}}/2)>0$. Summing from $t=0$ to $T-1$ gives
\[
E(\mathbf{X}^{(0)})-E(\mathbf{X}^{(T)})
\ge
c\sum_{t=0}^{T-1}\|\nabla E(\mathbf{X}^{(t)})\|_F^2.
\]
Using lower boundedness (Proposition~\ref{prop:coercive_base}) and letting $T\to\infty$ yields
$\sum_t\|\nabla E(\mathbf{X}^{(t)})\|_F^2<\infty$, so norms vanish: item (2).
For item (3), if $\mathbf{X}^{(t_k)}\to\mathbf{X}^\star$, then by continuity of $\nabla E$,
\[
\nabla E(\mathbf{X}^\star)=\lim_{k\to\infty}\nabla E(\mathbf{X}^{(t_k)})=\mathbf{0}.
\]
\end{proof}

\begin{theorem}[Rates]
\label{thm:rates_base}
Under the assumptions of Theorem~\ref{thm:gd_base}, let $E_{\inf} :=\inf_{\mathbf{X}}E_{\mathrm{base}}(\mathbf{X})$ and
$c:=\eta(1-\eta L_{\mathrm{lip}}/2)$.
Then for any $T\ge 1$,
\[
\min_{0\le t\le T-1}\|\nabla E_{\mathrm{base}}(\mathbf{X}^{(t)})\|_F^2
\le
\frac{E_{\mathrm{base}}(\mathbf{X}^{(0)})-E_{\inf}}{cT}.
\]
If $\beta\|\mathbf{M}\|^2<2$ (so $E_{\mathrm{base}}$ is $\mu$-strongly convex with $\mu=1-\beta\|\mathbf{M}\|^2/2$),
then for $\eta\in(0,1/L_{\mathrm{lip}}]$ and minimizer $\mathbf{X}^\star$,
\[
\|\mathbf{X}^{(t)}-\mathbf{X}^\star\|_F
\le
(1-\eta\mu)^t\|\mathbf{X}^{(0)}-\mathbf{X}^\star\|_F,
\]
and
\[
E_{\mathrm{base}}(\mathbf{X}^{(t)})-E_{\mathrm{base}}(\mathbf{X}^\star)
\le
(1-\eta\mu)^t\left(E_{\mathrm{base}}(\mathbf{X}^{(0)})-E_{\mathrm{base}}(\mathbf{X}^\star)\right).
\]
\end{theorem}

\begin{proof}
For the first claim, Theorem~\ref{thm:gd_base} gives
\[
E(\mathbf{X}^{(t+1)})\le E(\mathbf{X}^{(t)})-c\|\nabla E(\mathbf{X}^{(t)})\|_F^2.
\]
Summing from $t=0$ to $T-1$ and using $E(\mathbf{X}^{(T)})\ge E_{\inf}$,
\[
c\sum_{t=0}^{T-1}\|\nabla E(\mathbf{X}^{(t)})\|_F^2
\le
E(\mathbf{X}^{(0)})-E_{\inf}.
\]
Divide by $T$ and lower bound the average by the minimum.

For strong convexity, let
$\mathbf{X}^\star\in\operatorname*{arg\,min}_{\mathbf{X}} E(\mathbf{X})$ (unique).
By smoothness and strong convexity, for each $t$ there exists symmetric $\mathbf{A}^{(t)}$ with
$\mu\mathbf{I}\preceq \mathbf{A}^{(t)}\preceq L_{\mathrm{lip}}\mathbf{I}$ such that
\[
\nabla E(\mathbf{X}^{(t)})-\nabla E(\mathbf{X}^\star)=\mathbf{A}^{(t)}(\mathbf{X}^{(t)}-\mathbf{X}^\star).
\]
Since $\nabla E(\mathbf{X}^\star)=0$,
\[
\mathbf{X}^{(t+1)}-\mathbf{X}^\star
=
\left(\mathbf{I}-\eta\mathbf{A}^{(t)}\right)(\mathbf{X}^{(t)}-\mathbf{X}^\star).
\]
Therefore
\[
\|\mathbf{X}^{(t+1)}-\mathbf{X}^\star\|_F
\le
\max_{\lambda\in[\mu,L_{\mathrm{lip}}]}|1-\eta\lambda|\;\|\mathbf{X}^{(t)}-\mathbf{X}^\star\|_F.
\]
For $\eta\le 1/L_{\mathrm{lip}}$, the maximum is $1-\eta\mu$, giving the iterate rate.
For function values, combine
\[
E(\mathbf{X}^{(t+1)})-E(\mathbf{X}^\star)
\le
E(\mathbf{X}^{(t)})-E(\mathbf{X}^\star)-\eta\left(1-\frac{\eta L_{\mathrm{lip}}}{2}\right)\|\nabla E(\mathbf{X}^{(t)})\|_F^2,
\]
with gradient domination for $\mu$-strongly convex functions,
\[
\|\nabla E(\mathbf{X}^{(t)})\|_F^2
\ge
2\mu\left(E(\mathbf{X}^{(t)})-E(\mathbf{X}^\star)\right),
\]
and $\eta\le 1/L_{\mathrm{lip}}$ to conclude
\[
E(\mathbf{X}^{(t+1)})-E(\mathbf{X}^\star)
\le
(1-\eta\mu)\left(E(\mathbf{X}^{(t)})-E(\mathbf{X}^\star)\right).
\]
\end{proof}

\subsection{Fixed-Point Contraction and Damping}
\label{app:theory_fp}

Define
\[
[T(\mathbf{X})]_v=\mathbf{M}^\top\softmax(\beta\mathbf{M}\mathbf{x}_v)-2\lambda(\mathbf{L}\mathbf{X})_v.
\]

\begin{remark}[Gradient descent vs. fixed-point iteration]
\label{rem:gd_vs_fp}
It is useful to separate two schemes:
\begin{enumerate}
\item \textbf{Gradient descent}: $\mathbf{X}^{(t+1)}=\mathbf{X}^{(t)}-\eta\nabla E_{\mathrm{base}}(\mathbf{X}^{(t)})$.
Under $L_{\mathrm{lip}}$-smoothness and $\eta\in(0,2/L_{\mathrm{lip}})$, this guarantees monotone energy descent.
\item \textbf{Fixed-point iteration}: $\mathbf{X}^{(t+1)}=T(\mathbf{X}^{(t)})$.
This does not automatically imply descent; global convergence is guaranteed only if $T$ is a contraction.
\end{enumerate}
\end{remark}

\begin{theorem}[Fixed-point contraction condition]
\label{thm:fp_base}
The map $T$ is globally Lipschitz with
\[
\mathrm{Lip}(T)\le \rho:=\frac{\beta}{2}\|\mathbf{M}\|^2+2\lambda\|\mathbf{L}\|.
\]
If $\rho<1$, then $T$ is a contraction: it has a unique fixed point $\mathbf{X}^\star$ and
$\mathbf{X}^{(t+1)}=T(\mathbf{X}^{(t)})$ converges geometrically to $\mathbf{X}^\star$ from any initialization.
\end{theorem}

\begin{proof}
Write $T(\mathbf{X})=S(\mathbf{X})-2\lambda\mathbf{L}\mathbf{X}$.
By the proof of Proposition~\ref{prop:smoothness_base},
$S$ is $\frac{\beta}{2}\|\mathbf{M}\|^2$-Lipschitz, and
$\|\mathbf{L}(\mathbf{X}-\mathbf{Y})\|_F\le \|\mathbf{L}\|\,\|\mathbf{X}-\mathbf{Y}\|_F$.
Therefore
\[
\|T(\mathbf{X})-T(\mathbf{Y})\|_F
\le
\left(\frac{\beta}{2}\|\mathbf{M}\|^2+2\lambda\|\mathbf{L}\|\right)\|\mathbf{X}-\mathbf{Y}\|_F.
\]
Hence $\mathrm{Lip}(T)\le\rho$.
If $\rho<1$, Banach's fixed-point theorem gives existence and uniqueness of $\mathbf{X}^\star$ and
\[
\|\mathbf{X}^{(t)}-\mathbf{X}^\star\|_F\le \rho^t\|\mathbf{X}^{(0)}-\mathbf{X}^\star\|_F.
\]
\end{proof}

\begin{corollary}[Damped map]
\label{cor:fp_damped}
For $T_\alpha(\mathbf{X})=(1-\alpha)\mathbf{X}+\alpha T(\mathbf{X})$ with $\alpha\in(0,1]$,
\[
\mathrm{Lip}(T_\alpha)\le (1-\alpha)+\alpha\rho.
\]
Hence if $\rho<1$, every $\alpha\in(0,1]$ yields a contraction.
\end{corollary}

\begin{proof}
For any $\mathbf{X},\mathbf{Y}$,
\[
\|T_\alpha(\mathbf{X})-T_\alpha(\mathbf{Y})\|_F
\le
(1-\alpha)\|\mathbf{X}-\mathbf{Y}\|_F+\alpha\|T(\mathbf{X})-T(\mathbf{Y})\|_F,
\]
then apply Theorem~\ref{thm:fp_base}.
\end{proof}

\begin{corollary}[Convenient contraction condition for normalized Laplacian]
\label{cor:fp_normlap}
Since $\|\mathbf{L}\|\le 2$ for normalized $\mathbf{L}$, a sufficient condition for undamped fixed-point contraction is
\[
\frac{\beta}{2}\|\mathbf{M}\|^2+4\lambda<1.
\]
\end{corollary}

\begin{proof}
Immediate from Theorem~\ref{thm:fp_base} and $\|\mathbf{L}\|\le 2$.
\end{proof}

\subsection{Critical Points and Initialization}
\label{app:theory_critical}

\begin{proposition}[Critical points]
\label{prop:critical_base}
$\mathbf{X}^\star$ is first-order critical iff $\nabla E_{\mathrm{base}}(\mathbf{X}^\star)=0$.
Equivalent fixed-point form:
\[
\mathbf{X}^\star=T(\mathbf{X}^\star).
\]
If $\nabla^2 E_{\mathrm{base}}(\mathbf{X}^\star)\succ 0$, then $\mathbf{X}^\star$ is a strict local minimizer;
if $\nabla^2 E_{\mathrm{base}}(\mathbf{X}^\star)$ has a negative eigenvalue, then $\mathbf{X}^\star$ is a strict saddle.
\end{proposition}

\begin{proof}
First-order: by definition, $\mathbf{X}^\star$ is critical iff $\nabla E_{\mathrm{base}}(\mathbf{X}^\star)=0$.
Using Proposition~\ref{prop:gradient_formula}, this is equivalent to $\mathbf{X}^\star=T(\mathbf{X}^\star)$.

For second-order classification, Taylor-expand around $\mathbf{X}^\star$:
\[
E(\mathbf{X}^\star+\Delta)=E(\mathbf{X}^\star)+\langle \nabla E(\mathbf{X}^\star),\Delta\rangle+\frac12\langle \Delta,\nabla^2E(\mathbf{X}^\star)\Delta\rangle+o(\|\Delta\|_F^2).
\]
At a critical point the linear term is zero.
If $\nabla^2E(\mathbf{X}^\star)\succ 0$, the quadratic term is positive for all sufficiently small nonzero $\Delta$, so $\mathbf{X}^\star$ is a strict local minimizer.
If $\nabla^2E(\mathbf{X}^\star)$ has a negative eigenvalue, pick $\Delta=\epsilon\mathbf{v}$ along a corresponding eigenvector and small $\epsilon>0$ to make the quadratic term negative, yielding a strict saddle.
\end{proof}

\begin{proposition}[When initialization does not matter]
\label{prop:init_base}
If $\beta\|\mathbf{M}\|^2<2$, gradient descent converges to the unique global minimizer from any initialization.
If $\rho<1$ in Theorem~\ref{thm:fp_base}, fixed-point iteration converges to the unique fixed point from any initialization.
\end{proposition}

\begin{proof}
The first claim is Theorem~\ref{thm:rates_base} in the strongly convex regime.
The second claim follows from Banach's fixed-point theorem.
\end{proof}

\begin{remark}[When initialization can matter]
Outside the strongly convex/contraction regimes, $E_{\mathrm{base}}$ can be nonconvex with multiple stationary points.
Then basin geometry and algorithmic choices (step size, damping, stopping) can change which stationary point is reached.
\end{remark}

\end{document}